\documentclass{article}

\newif\ifsup\supfalse
\suptrue

\usepackage{microtype}
\usepackage{graphicx}

\usepackage{subfigure}
\usepackage{hyperref}
\usepackage[accepted]{icml2019}

\usepackage{amsmath}\allowdisplaybreaks
\usepackage{amssymb}
\usepackage{amsthm}
\usepackage{bbm}
\usepackage{bm}
\usepackage{tabularx}
\usepackage{dsfont}
\usepackage{enumerate}
\usepackage{graphicx}
\usepackage{wrapfig}
\usepackage{mathtools}
\usepackage[utf8]{inputenc} 
\usepackage[T1]{fontenc} 
\definecolor{dkblue}{cmyk}{1,.54,.04,.19} 
\usepackage{url} 
\usepackage{booktabs} 
\usepackage{amsfonts} 
\usepackage{nicefrac} 
\usepackage{microtype} 
\usepackage{xspace}
\usepackage{enumitem}
\usepackage[capitalize]{cleveref}

\newif\iftikz\tikztrue
\iftikz
\usepackage{tikz}
\usetikzlibrary{shapes,shapes.misc,positioning}
\usetikzlibrary{patterns,arrows,decorations.pathreplacing}
\fi

\usepackage[backgroundcolor=White,textwidth=1.0\marginparwidth,disable]{todonotes}
\newcommand{\todoc}[2][]{\todo[color=Cyan!20,size=\tiny,#1]{Cs: #2}} 

\newcommand\numberthis{\addtocounter{equation}{1}\tag{\theequation}}

\newtheorem{theorem}{Theorem}

\newtheorem{lemma}{Lemma}

\newtheorem{definition}{Definition}
\newtheorem{assumption}{Assumption}

\newtheorem{remark}{Remark}

\newcommand{\cA}{\mathcal{A}}
\newcommand{\cE}{\mathcal{E}}

\newcommand{\bbP}{\mathbb{P}}
\newcommand{\R}{\mathbb{R}}
\newcommand{\Prob}[1]{\bbP\left(#1\right)}

\let\tmp\epsilon
\let\epsilon\varepsilon
\let\varepsilon\tmp

\newcommand{\cL}{\mathcal L}
\newcommand{\cK}{\mathcal K}

\newcommand{\cD}{\mathcal D}
\newcommand{\cT}{\mathcal T}
\newcommand{\abs}[1]{\left|#1\right|}

\newcommand{\E}[1]{\mathbb{E} \left[#1\right]}

\newcommand{\norm}[1]{\left\|#1\right\|}

\newcommand{\sind}[1]{\mathds{1}_{#1}}

\newcommand{\rank}{\mathrm{rank}}

\newcommand{\set}[1]{\left\{#1\right\}}

\newcommand{\ip}[1]{\langle #1 \rangle}
\newcommand{\ceil}[1]{\left\lceil #1 \right\rceil}

\DeclareMathOperator*{\argmax}{arg\,max\,}

\mathchardef\mhyphen="2D

\newcommand{\gopt}{\operatorname{\textsc{Gopt}}}

\newcommand{\toprank}{{\tt TopRank}\xspace}
\newcommand{\rrank}{{\tt RecurRank}\xspace}
\newcommand{\cascadelinucb}{{\tt CascadeLinUCB}\xspace}

\newcommand{\EE}{\mathbb{E}}
\newcommand{\PP}{\mathbb{P}}
\newcommand{\RR}{\mathbb{R}}

\icmltitlerunning{Online Learning to Rank with Features}
\begin{document}

\twocolumn[
\icmltitle{Online Learning to Rank with Features}

\icmlsetsymbol{equal}{*}

\begin{icmlauthorlist}
\icmlauthor{Shuai Li}{cuhk}
\icmlauthor{Tor Lattimore}{deepmind}
\icmlauthor{Csaba Szepesv\'ari}{deepmind}
\end{icmlauthorlist}

\icmlaffiliation{cuhk}{The Chinese University of Hong Kong}
\icmlaffiliation{deepmind}{DeepMind}

\icmlcorrespondingauthor{Shuai Li}{shuaili@cse.cuhk.edu.hk}
\icmlcorrespondingauthor{Tor Lattimore}{lattimore@google.com}
\icmlcorrespondingauthor{Csaba Szepesv\'ari}{szepi@google.com}

\icmlkeywords{}

\vskip 0.3in
]

\printAffiliationsAndNotice{}

\begin{abstract}
We introduce a new model for online ranking in which the click probability factors into an examination and attractiveness function and the attractiveness function is a linear function of a feature vector and an unknown parameter. Only relatively mild assumptions are made on the examination function. A novel algorithm for this setup is analysed, showing that the dependence on the number of items is replaced by a dependence on the dimension, allowing the new algorithm to handle a large number of items. When reduced to the orthogonal case, the regret of the algorithm improves on the state-of-the-art.
\end{abstract}

\section{Introduction}

Let $\cL$ be a large set of items to be ranked. For example, a database of movies, news articles or search results.
We consider a sequential version of the ranking problem where in each round
the learner chooses an ordered list of $K$ distinct items from $\cL$ to show the user. 
We assume the feedback comes in the form of clicks and the learner's objective is
to maximise the expected number of clicks over $T$ rounds. 
Our focus is on the case where $\cL$ is large (perhaps millions) and $K$ is relatively small (fifty or so).
There are two main challenges that arise in online ranking problems:

(a) The number of rankings grows exponentially in $K$, which makes learning one parameter for each ranking a fruitless endeavour.
Click models may be used to reduce the dimensionality of the learning problem, but balancing generality of the model with learnability is a serious challenge. The majority of previous
works on online learning to rank have used unstructured models, which are not well suited to our setting where $\cL$ is large. 

(b) Most click models depend on an unknown attractiveness function that endows the item set with an order. This yields a model with at least $|\cL|$ parameters,
which is prohibitively large in the applications we have in mind. 

The first challenge is tackled by adapting the flexible click models introduced in \cite{ZTG17,LKLS18ranking} to our setting.
For the second we follow previous works on bandits with large action sets by assuming the attractiveness function can be
written as a linear function of a relatively small number of features.

\paragraph{Contribution}
We make several contributions:
\begin{itemize}
\item A new model for 
ranking problems with features is proposed that generalises previous work \cite{LWZC16,ZNK16,LLZ18} by relaxing the relatively restrictive assumptions 
on the probability that a user clicks on an item. The new model is strictly more robust than previous works focusing on regret analysis for large item sets.
\item We introduce a novel polynomial-time algorithm called \rrank{}. The algorithm operates recursively over an increasingly fine set of partitions of $[K]$. Within each
partition the algorithm balances exploration and exploitation, subdividing the partition once it becomes sufficiently certain about the suboptimality
of a subset of items.
\item A regret analysis shows that the cumulative regret of \rrank{} is at most
$R_T = O(K \sqrt{d T \log(LT})$, where $K$ is the number of positions, $L$ is the number of items
and $d$ is the dimension of the feature space. Even in the 
non-feature case where $L = d$ this improves on the state-of-the-art by a factor of $\sqrt{K}$.
\end{itemize}
A comparison with most related work is shown in Table \ref{table:full comparisons with related work}.

\begin{table*}
\label{table:full comparisons with related work}
\caption{This table compares settings and regret bounds of most related works on online learning to rank. $T$ is the number of total rounds, $K$ is the number of positions, $L$ is the number of items and $d$ is the feature space dimension. $\Delta$ is the minimal gap between the expected click rate of the best items and the expected click rate of the suboptimal items. }
\centering
\small
\renewcommand{\arraystretch}{1}
\begin{tabularx}{\textwidth}{@{}XlXp{3cm}}
\toprule
&Context &Click Model & Regret \\
\midrule
\citet{KSWA15}&- &Cascade Model (CM) &$\displaystyle \Theta\left(\frac{L}{\Delta}\log(T)\right)$\\
\midrule
\citet{LWZC16} \newline \citet{ZNK16} \newline \citet{li18clustering} & (Generalised) Linear Form &CM &$\displaystyle O\left(d\sqrt{TK}\log(T)\right)$ \\
\midrule
\citet{KKS16} &- &Dependent Click Model (DCM) &$\displaystyle \Theta\left(\frac{L}{\Delta}\log(T)\right)$\\
\midrule
\citet{LLZ18} & Generalised Linear Form &DCM &$\displaystyle O\left(dK\sqrt{TK}\log(T)\right)$\\
\midrule
\citet{LVC16} &- & Position-Based Model (PBM) with known position bias  & $\displaystyle O\left(\frac{L}{\Delta}\log(T)\right)$\\
\midrule
\citet{ZTG17} &- &General Click Model &$\displaystyle O\left(\frac{K^3L}{\Delta}\log(T)\right)$\\
\midrule
\citet{LKLS18ranking} &- &General Click Model & $\displaystyle O\left(\frac{KL}{\Delta}\log(T)\right)$ \newline $\displaystyle  O\left(\sqrt{K^3 L T \log(T)}\right)$ \newline $\displaystyle \Omega\left(\sqrt{KLT}\right)$ \\
\midrule
Ours &Linear Form &General Click Model &$\displaystyle O\left(K \sqrt{dT\log(LT)}\right)$\\
\bottomrule
\end{tabularx}
\vspace{-0.3cm}
\end{table*}

\paragraph{Related work}

Online learning to rank has seen an explosion of research in the last decade and there are multiple ways 
of measuring the performance of an algorithm. One view is that the clicks themselves should be maximised, which we take in this article. 
An alternative is to assume an underlying relevance of all items in a ranking that is never directly observed, but can be inferred
in some way from the observed clicks. In all generality this latter setting falls into the partial monitoring framework \cite{Rus99}, but has
been studied in specific ranking settings \citep[and references therein]{Cha16}. See the article by \citet{HW11} for more discussion on
various objectives.

Maximising clicks directly is a more straightforward objective because clicks are an observed quantity. 
Early work was empirically focused. For example, \citet{LC10} propose a modification of 
LinUCB for contextual ranking and \citet{CH15} modify the optimistic algorithms
for linear bandits.
These algorithms do not come with theoretical guarantees, however. There has recently been significant effort towards
designing theoretically justified algorithms in settings of increasing complexity \cite{KSWA15,CMP15,ZNK16,KKS16,LVC16}.
These works assume the user's clicks follow a click model that connects properties of the shown ranking to the probability that a user clicks
on an item placed in a given position. For example, in the document-based model it is assumed that the probability that the user clicks on a shown
item only depends on the unknown attractiveness of that item and not its position in the ranking or the other items.
Other simple models include the position-based, cascade and dependent click models. For a survey of click models see \cite{CMR15}. 

As usual, however, algorithms designed for specific models are brittle when the modelling assumptions are not met.
Recent work has started to relax the strong assumptions by making the observation that in all of the above click models the
probability of a user clicking on an item can be written as the product of the item's inherent attractiveness and the probability that the user 
examines its position in the list. \citet{ZTG17} 
use a click model where this decomposition is kept, but the assumption on how the examination probability of a position depends on the 
list is significantly relaxed. This is relaxed still further by \citet{LKLS18ranking} who avoid the factorisation assumption by making assumptions directly
on the click probabilities, but the existence of an attractiveness function remains. 

The models mentioned in the last paragraph do not make assumptions on the attractiveness function, which means the regret depends badly on the size of $\cL$.
Certain simple click models have assumed the attractiveness function is a linear function of an item's features and the resulting algorithms are suitable for large action sets.
This has been done for the cascade model \cite{LWZC16} and the dependent-click model \cite{LLZ18}.
While these works are welcomed, the strong assumptions leave a lingering doubt that perhaps the models may not be a good fit for practical problems.
Of course, our work is closely related to stochastic linear bandits, first studied by \citet{AL99} 
and refined by \citet{Aue02,AST11,VMKK14} and many others.

Ranking has also been examined in an adversarial framework by \citet{RKJ08}.
These settings are most similar to the stochastic position-based and document-based models, but with the additional robustness bought by the adversarial framework.
Another related setup is the rank-$1$ bandit problem in which the learner should choose just one of $L$ items to place in one of $K$ positions. 
For example, the location of a billboard with the budget to place only one. These setups have a lot in common with the present one, but cannot be directly applied to ranking problems.
For more details see \cite{KKS17,KKS17b}.

Finally, we note that some authors do not assume an ordering of the item set provided by an attractiveness function.
The reader is referred to the work by \citet{SRG13} (which is a follow-up work to \citet{RKJ08}) where the learner's objective is to maximise the probability that a user clicks on 
\emph{any} item, rather than rewarding multiple clicks. This model encourages diversity and provides an interesting alternative approach.

\section{Preliminaries}

\paragraph{Notation}
Let $[n] = \{1,2,\ldots,n\}$ denote the first $n$ natural numbers. 
Given a set $X$ the indicator function is $\sind{X}$.
For vector $x \in \R^d$ and positive definite matrix $V \in \R^{d\times d}$ we let $\norm{x}_V^2 = x^\top V x$.
The Moore-Penrose pseudoinverse of a matrix $V$ is $V^\dagger$.

\vspace{-0.1cm}
\paragraph{Problem setup}
Let $\cL \subset \R^d$ be a finite set of items, $L = |\cL|$ and $K>0$ a natural number, denoting the number of positions.
A ranking is an injective function from $[K]$, the set of positions, to $\cL$ 
and the set of all rankings is denoted by $\Sigma$.
We use uppercase letters like $A$ to denote rankings in $\Sigma$ and lowercase letters $a,b$ to denote items in $\cL$.
The game proceeds over $T$ rounds. In each round $t\in [T]$ the learner chooses a ranking $A_t \in \Sigma$ and subsequently receives feedback in
the form of a vector $C_t \in \{0,1\}^K$ where $C_{tk} = 1$ if the user clicked on the $k$th position.
We assume that the conditional distribution of $C_t$ only depends on $A_t$, which means there exists an unknown function $v : \Sigma \times [K] \to [0,1]$ such that for all $A \in \Sigma$ and $k \in [K]$,
\begin{align}
\Prob{C_{tk} = 1 \mid A_t = A} = v(A, k)\,.
\label{eq:click}
\end{align}

\begin{remark}
We do not assume conditional independence of $(C_{tk})_{k=1}^K$.
\end{remark}

In all generality the function $v$ has $K|\Sigma|$ parameters, which is usually impractically large to learn in any reasonable time-frame. 
A click model corresponds to making assumptions on $v$ that reduces the statistical complexity of the learning problem. 
We assume a factored model:
\begin{align}
v(A, k) = \chi(A, k) \alpha(A(k))\,,
\label{eq:clickfactor}
\end{align}
where $\chi : \Sigma \times [K] \to [0,1]$ is called the examination probability and $\alpha : \cL \to [0,1]$ is the attractiveness function.
We assume that attractiveness is linear in the action, which means there exists an unknown $\theta_\ast \in \R^d$ such that
\begin{align}
\alpha(a) = \ip{a, \theta_{\ast}} \quad \text{for all } a \in \cL\,.
\label{eq:clicklinear}
\end{align}

Let $a_k^\ast$ be the $k$-th best item sorted in order of decreasing attractiveness. Then let $A^\ast = \left(a_1^\ast, \ldots, a_K^\ast\right)$. 
In case of ties the choice of $A^\ast$ may not be unique. All of the results that follow hold for any choice.

The examination function satisfies three additional assumptions.
The first says the examination probability of position $k$ only depends on the identity of the first $k-1$ items and not their order:

\begin{assumption}\label{ass:perm}
$\chi(A, k) = \chi(A', k)$ for any $A, A' \in \Sigma$ with $A([k-1]) = A'([k-1])$. 
\end{assumption}

The second assumption is that the examination probability on any ranking is monotone decreasing in $k$:

\begin{assumption}\label{ass:decrease}
$\chi(A, k+1) \leq \chi(A, k)$ for all $A \in \Sigma$ and $k \in [K-1]$.
\end{assumption}

The third assumption is that the examination probability on ranking $A^\ast$ is minimal:

\begin{assumption}\label{ass:min}
$\chi(A, k) \ge \chi(A^\ast, k) =: \chi_k^\ast$ for all $A \in \Sigma$ and $k \in [K]$.
\end{assumption}

All of these assumptions are satisfied by many standard click models, including the document-based, position-based and cascade models.
These assumptions are strictly weaker than those made by \citet{ZTG17}
 and orthogonal to those by \citet{LKLS18ranking} as we discuss it in \cref{sec:discussion}.

\paragraph{The learning objective}
We measure the performance of our algorithm in terms of the cumulative regret,
which is
\begin{align*}
R_T = T \sum_{k=1}^K v(A^\ast, k) - \E{\sum_{t=1}^T \sum_{k=1}^K v(A_t, k)} \,.
\end{align*}

\begin{remark}
The regret is defined relative to $A^\ast$, but our
assumptions do not imply that
\begin{align}
A^\ast \in \argmax_{A \in \Sigma} \sum_{k=1}^K v(A, k)\,. \label{eq:maximizer}
\end{align}
The assumptions in all prior work in Table~\ref{table:full comparisons with related work} either directly or indirectly
ensure that (\ref{eq:maximizer}) holds. Our regret analysis does not rely on this, so we do not assume it. Note, however, 
that the definition of regret is most meaningful when (\cref{eq:maximizer}) approximately holds.
\end{remark}

\paragraph{Experimental design}
Our algorithm makes use of an exploration `spanner' 
that approximately minimises the covariance of the least-squares estimator.
Given an arbitrary finite set of vectors $X = \{x_1,\ldots,x_n\} \subset \R^d$ and distribution $\pi : X \to [0,1]$ let $Q(\pi) = \sum_{x \in X} \pi(x) xx^\top$. 
By the Kiefer--Wolfowitz theorem \cite{KW60} there exists a $\pi$ called the $G$-optimal design such that
\begin{align}
\max_{x \in X} \norm{x}_{Q(\pi)^{\dagger}}^2 \leq d\,.
\label{eq:gopt}
\end{align}
As explained in Chap. 21 of \citep{LS18book}, $\pi$ may be chosen so that $|\{x : \pi(x) > 0\}| \leq d(d+1)/2$.
A $G$-optimal design $\pi$ for $X$ has the property that if each element $x\in X$ is observed 
$n \pi(x)$ times for some $n$ large enough, 
the value estimate obtained via least-squares 
will have its maximum uncertainty over the items minimised. 
\todoc{We could be a bit more specific. In fact, we could start with the goal of optimal design.}
Given a finite (multi-)set of vectors $X \subset \R^d$ we let $\gopt(X)$ denote a $G$-optimal design distribution.
Methods from experimental design have been used for pure exploration in linear bandits \cite{SLM14,XHS17} and also finite-armed linear bandits \citep[Chap. 22]{LS18book} as
well as adversarial linear bandits \cite{BCK12}.

\newcommand{\minels}{\operatorname{minelements}}
\newcommand{\alghead}[1]{\underline{\textsc{#1}} \\[0.1cm]}

\vspace{-0.3cm}
\section{Algorithm}
\label{sec:Algorithm}

The new algorithm is called \rrank (`recursive ranker'). 
The algorithm maintains a partition of the $K$ positions into intervals. Associated with each interval is an integer-valued `phase number' 
and an ordered set of items, which has the same size as the interval for all but the last interval (containing position $K$). 
Initially the partition only contains one interval that is associated with all the items and phase number $\ell = 1$.

At any point in time, \rrank works in parallel on all intervals.
Within an interval associated with phase number $\ell$, the algorithm balances exploration and exploitation while determining the relative attractiveness of the items to accuracy $\Delta_\ell = 2^{-\ell}$. 
To do this, items are placed in the first position of the interval in proportion to an experimental design. The remaining
items are placed in order in the remaining positions. Once sufficient data is collected, the interval is divided into a collection of subintervals
and the algorithm is restarted on each subinterval with the phase number increased.

The natural implementation of the algorithm maintains a list of partitions and associated items. In each round it iterates over the partitions and makes assignments of the items within
each partition. The assignments are based on round-robin idea using an experimental design, which means the algorithm needs to keep track of how often each item has been placed in
the first position. This is not a problem
from an implementation perspective, but stateful code is hard to interpret in pseudocode. We provide a recursive implementation that describes the assignments made within each 
interval and the rules for creating a new partition.
A flow chart depicting the operation of the algorithm is given in \cref{fig:flow chart for algo}. The code is provided in the supplementary material.

\begin{algorithm}[thb!]
\begin{algorithmic}[1]
\STATE \label{alg:main:input} \textbf{Input: } Phase number $\ell$ and \\ $\cA = (a_1,a_2,\ldots)$ and $\cK = (k,k+1,\ldots,k+m-1)$ 
\STATE \label{alg:main:g-optimal design} 
Find a $G$-optimal design $\pi = \gopt(\cA)$ 
\STATE \label{alg:main:define T(a)} 
Let $\Delta_\ell = 2^{-\ell}$ and 
\begin{align}
  T(a) = \ceil{\frac{d\,\pi(a)}{2\Delta_\ell^2}\log\left(\frac{|\cA|}{\delta_\ell}\right)}
  \label{eq:allocchoice}
\end{align}
This instance will run $\sum_{a \in \cA} T(a)$ times
\STATE 
\label{alg:main:select} 
Select each item $a\in \cA$ exactly $T(a)$ times at position $k$ 
  and put available items in $\{a_1,\ldots,a_m\}$ sequentially in positions $\{k+1,\ldots,k+m-1\}$ and receive feedbacks (synchronized by a global clock).
\STATE 
\label{alg:main:compute theta hat} 
Let $\cD = \{(\beta_1, \zeta_1), \ldots\}$ be the multiset of item/clicks from position $k$ 
  and compute 
\begin{align} 
\hat \theta &= V^\dagger S \quad \text{ with } \label{eq:lse} \\
& \qquad V = \sum_{(\beta, \zeta) \in \cD} \beta \beta^{\top} \text{ and } S = \sum_{(\beta, \zeta) \in \cD} \beta \zeta \nonumber
\end{align}
\STATE Let $a^{(1)},a^{(2)},\ldots,a^{(|\cA|)}$ be an ordering $\cA$ such that
\begin{align*}
\epsilon_i = \ip{\hat \theta, a^{(i)} - a^{(i+1)}} \geq 0 \text{ for all } 1 \leq i < |\cA|
\end{align*}
and set $\epsilon_{|\cA|}=2\Delta_\ell$
\STATE Let $(u_1,\ldots,u_p) = (i \in [|\cA|] : \epsilon_i \geq 2\Delta_\ell)$ and $u_0 = 0$
\begin{align*}
\cA_i &= (a^{(u_{i-1}+1)},\ldots,a^{(u_{i})}) \\
\cK_i &= (k+u_{i-1},\ldots,k+\min(m,u_i)-1)
\end{align*}
\STATE \label{alg:main:elim}
For each $i \in [p]$ such that $k+u_{i-1}\le k+m-1$ call \rrank$(\ell+1, \cA_i, \cK_i)$ on separate threads
\end{algorithmic}
\caption{\rrank}\label{alg:main}
\end{algorithm}

The pseudocode of the core subroutine on each interval is given in \cref{alg:main}. The subroutine accepts as input (1) the phase number $\ell$, (2) the positions of the interval $\cK \subseteq [K]$ 
and (3) an ordered list
of items, $\cA$. The phase number determines the length of the experiment and the target precision. 
The ordering of the items in $\cA$ is arbitrary in the initial partition (when $\ell = 1$). When $\ell > 1$ the ordering is determined by the empirical estimate of attractiveness in
the previous experiment, which is crucial for the analysis.
The whole algorithm is started by calling $\rrank(1,\cL,(1,2,\dots,K))$ where the order of $\cL$ is random.
The algorithm is always instantiated with parameters that satisfy $|\cA| \geq |\cK|=m$. Furthermore, $|\cA| > |\cK|$ is only possible when $K \in \cK$.

\begin{figure*}[thb!]
\centering
\begin{tikzpicture}[scale=0.7,font=\small]

\draw (0,0) rectangle (1,-4);
\draw[dashed] (0,-0.5) -- (1,-0.5);
\node at (-0.2, -0.25) {$1$};
\node at (-0.2, -3.75) {$8$};
\node at (0.5, 0.25) {$\ell=1$};
\node at (1.5, 0.75) {$\cA$};
\node at (1.5, 0.25) {$||$};
\node at (1.5, -0.25) {$\overbrace{a_1}$};
\node at (1.5, -1.5) {$\cdot$};
\node at (1.5, -2) {$\cdot$};
\node at (1.5, -2.5) {$\cdot$};
\node at (1.5, -3.75) {$a_8$};
\node at (1.5, -4.25) {$\cdot$};
\node at (1.5, -4.5) {$\cdot$};
\node at (1.5, -4.75) {$\cdot$};
\node at (1.5, -5.25) {$\underbrace{a_{50}}$};
\node[rounded rectangle,draw,minimum width=5.7cm,minimum height=2.2cm,dotted,rotate=90] (r1) at (0.9,-2.3) {};

\draw[-latex] (2.3,-2.3) -- (4.55,1.1);
\draw (5,1.5) rectangle (6,0);
\draw[dashed] (5,1) -- (6,1);
\node at (4.8, 1.25) {$1$};
\node at (4.8, 0.25) {$3$};
\node at (5.5, 1.75) {$\ell=2$};
\node at (6.5, 2.25) {$\cA$};
\node at (6.5, 1.75) {$||$};
\node at (6.5, 1.25) {$\overbrace{a_1}$};
\node at (6.5, 0.75) {$\vdots$};
\node at (6.5, 0.25) {$\underbrace{a_3}$};
\node[rounded rectangle,draw,minimum width=2.8cm,minimum height=2.1cm,dotted,rotate=90] (r1) at (5.9,1.1) {};

\draw[-latex] (2.3,-2.3) -- (4.5,-4.5);
\draw (5,-3) rectangle (6,-5.5);
\draw[dashed] (5,-3.5) -- (6,-3.5);
\node at (5.5, -2.75) {$\ell=2$};
\node at (6.5, -2.25) {$\cA$};
\node at (6.5, -2.75) {$||$};
\node at (4.8, -3.25) {$4$};
\node at (4.8, -5.25) {$8$};
\node at (6.5, -3.25) {$\overbrace{a_4}$};
\node at (6.5, -3.75) {$\cdot$};
\node at (6.5, -4.25) {$\cdot$};
\node at (6.5, -4.75) {$\cdot$};
\node at (6.5, -5.25) {$a_8$};
\node at (6.5, -5.75) {$\cdot$};
\node at (6.5, -6) {$\cdot$};
\node at (6.5, -6.25) {$\cdot$};
\node at (6.5, -6.75) {$\underbrace{a_{25}}$};
\node[rounded rectangle,draw,minimum width=4.5cm,minimum height=2.25cm,dotted,rotate=90] (r1) at (5.9,-4.5) {};

\draw[-latex] (7.2,1.1) -- (8.55,1.1);
\draw (9,1.5) rectangle (10,0);
\draw[dashed] (9,1) -- (10,1);
\node at (8.8, 1.25) {$1$};
\node at (8.8, 0.25) {$3$};
\node at (9.5, 1.75) {$\ell=3$};
\node at (10.5, 2.25) {$\cA$};
\node at (10.5, 1.75) {$||$};
\node at (10.5, 1.25) {$\overbrace{a_1}$};
\node at (10.5, 0.75) {$\vdots$};
\node at (10.5, 0.25) {$\underbrace{a_3}$};
\node[rounded rectangle,draw,minimum width=2.8cm,minimum height=2.1cm,dotted,rotate=90] (r1) at (9.9,1.1) {};

\draw[-latex] (11.2,1.1) -- (15.5,1.1);
\node at (15.75, 1.1) {$\cdot$};
\node at (16, 1.1) {$\cdot$};
\node at (16.25, 1.1) {$\cdot$};

\draw[-latex] (7.25,-4.5) -- (11.55,-1.6);
\draw (12,-1.5) rectangle (13,-2.5);
\draw[dashed] (12,-2) -- (13,-2);
\node at (12.5, -1.25) {$\ell=3$};
\node at (13.5, -0.75) {$\cA$};
\node at (13.5, -1.25) {$||$};
\node at (11.8, -1.75) {$4$};
\node at (11.8, -2.25) {$5$};
\node at (13.5, -1.75) {$\overbrace{a_4}$};
\node at (13.5, -2.25) {$\underbrace{a_5}$};
\node[rounded rectangle,draw,minimum width=2cm,minimum height=2.1cm,dotted,rotate=90] (r1) at (12.9,-1.6) {};

\draw[-latex] (14.2,-1.6) -- (15.5,-1.6);
\node at (15.75, -1.6) {$\cdot$};
\node at (16, -1.6) {$\cdot$};
\node at (16.25, -1.6) {$\cdot$};

\draw[-latex] (7.25,-4.5) -- (11.55,-5.5);
\draw (12,-4.5) rectangle (13,-6);
\draw[dashed] (12,-5) -- (13,-5);
\node at (12.5, -4.25) {$\ell=3$};
\node at (13.5, -3.75) {$\cA$};
\node at (13.5, -4.25) {$||$};
\node at (11.8, -4.75) {$6$};
\node at (11.8, -5.75) {$8$};
\node at (13.5, -4.75) {$\overbrace{a_6}$};
\node at (13.5, -5.25) {$\vdots$};
\node at (13.5, -5.75) {$a_8$};
\node at (13.5, -6.25) {$\cdot$};
\node at (13.5, -6.5) {$\cdot$};
\node at (13.5, -6.75) {$\cdot$};
\node at (13.5, -7.25) {$\underbrace{a_{12}}$};
\node[rounded rectangle,draw,minimum width=3.7cm,minimum height=2.1cm,dotted,rotate=90] (r1) at (12.9,-5.5) {};

\draw[-latex] (14.2,-5.5) -- (15.5,-5.5);
\node at (15.75, -5.5) {$\cdot$};
\node at (16, -5.5) {$\cdot$};
\node at (16.25, -5.5) {$\cdot$};

\draw[-latex] (-1,-8.5) -- (17,-8.5);
\node at (16.75, -8.75) {$t$};
\draw[dashed] (1,-5.75) -- (1,-8.5);
\node at (1,-6) {\color{red} Instance 1};
\node at (1, -8.75) {$1$};
\draw[dashed] (5.9,-0.5) -- (5.9,-1.9);
\node at (5.9,-0.75) {\color{red} Instance 2};
\draw[dashed] (5.9,-7.25) -- (5.9,-8.5);
\node at (5.9,-7.5) {\color{red} Instance 3};
\node at (5.9, -8.75) {$t_1$};
\draw[dashed] (9.9,-0.5) -- (9.9,-8.5);
\node at (9.9,-0.75) {\color{red} Instance 4};
\node at (9.9, -8.75) {$t_2$};
\draw[dashed] (12.9,-2.85) -- (12.9,-3.4);
\node at (12.9,0) {\color{red} Instance 5};
\draw[dashed] (12.9,-7.75) -- (12.9,-8.5);
\node at (12.9,-8) {\color{red} Instance 6};
\node at (12.9, -8.75) {$t_3$};

\end{tikzpicture}
\caption{A flow chart demonstration for the algorithm. Each dotted circle represents a subinterval and runs an instance of \cref{alg:main}. The dashed line denotes the first position for each interval.}
\vspace{-0.3cm}
\label{fig:flow chart for algo}
\end{figure*}
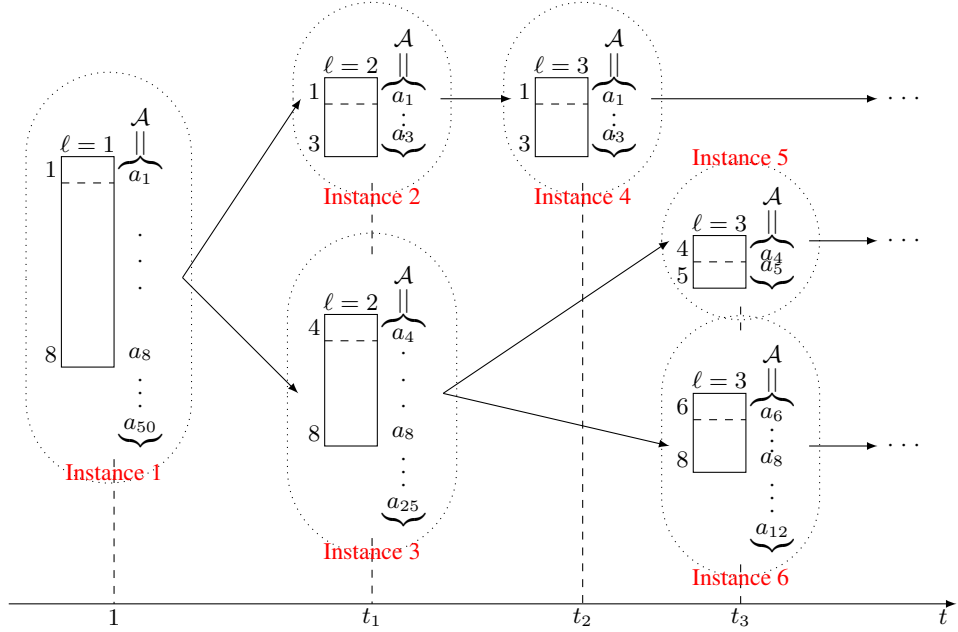

The subroutine learns about the common unknown parameter vector by placing items in the first position of the interval 
in proportion to a $G$-optimal design for the available items. The remaining items in $\cA$ are placed in order into the remaining positions (Line 4).
This means that each item $a \in \cA$ is placed exactly $T(a)$ times in the first position $k$ of the interval. The choice of $T(a)$ is based on the 
phase number $\ell$ and the $G$-optimal design $\pi$ over $\cA$ (Line 2). Note $T(a) = 0$ if $\pi(a)=0$. For example, if $\cA=(a_1,\ldots,a_m)$ and $a_3$ 
is placed at the first position, then the rest positions are filled in as $a_1,a_2,a_4,a_5,\ldots,a_m$.
The subroutine runs for $\sum_{a\in \cA} T(a)$ rounds.
The $G$-optimal design means that the number of rounds required to estimate the value of each item to a fixed precision depends only logarithmically on the number of items.
Higher phase number means longer experiment and also higher target precision.

Once all arms $a \in \cA$ have been placed in the first position of the interval $T(a)$ times, \rrank estimates the attractiveness of the items in $\cA$ 
using a least-squares estimator based on the data collected from the first position (Line 5). 
The items are then ordered based on their estimated attractiveness. The subroutine then partitions the ordered items
when the difference between estimated attractiveness of consecutive items is sufficiently large (Line 7).
Finally the subroutine recursively calls \rrank on each partition for which there are positions available with an increased phase number with items sorted according to their empirical attractiveness (Line 8).

\begin{remark}
Items are eliminated entirely if at the end of a subroutine a partition is formed for which there are no available positions.
For example, consider the first instantiation of \rrank with $\ell = 1$ and $\cK = [K]$ and $\cA = \cL$.
Suppose the observed data is such that $p = 2$ and $u_1 \geq K$, then items $a^{(u_1+1)},a^{(u_1+2)},\ldots,a^{(u_2)}$ will be discarded because the starting position 
of the second partition would be larger than $K$.
\end{remark}

\begin{remark}
The least-squares estimator $\hat \theta$ defined in \cref{eq:lse} actually 
does not have expectation $\theta$, which means the algorithm is not really estimating attractiveness. Our assumptions ensure that the expectation of $\hat \theta$ is proportional to $\theta$, however, 
which is sufficient for our analysis. This is the reason for only using the first position within an interval for estimation.
\end{remark}

\begin{remark}
The subroutine only uses data collected during its own run. Not doing this would introduce bias that may be hard to control. 
\end{remark}

\vspace{-0.3cm}
In \cref{fig:flow chart for algo}, the algorithm starts with Instance 1 of phase number $\ell=1$, all items and all positions. At time $t_1$, Instance 1 splits 
into two, each with an increased phase number $\ell=2$. Instance 2 contains $3$ items and $3$ positions and Instance 3 contains $5$ positions but $22$ items.
The remaining items have been eliminated.
At time $t_2$, Instance 2 finishes running but has no split, so it calls Instance 4 with the same items, same positions but increased phase 
number $\ell=3$. During time $t_1$ to $t_2$, Instance 2 and Instance 3 run in parallel and recommend lists together; during time $t_2$ to $t_3$, Instances 3 and 4 
run in parallel and recommend lists together. At time $t_3$, Instance 3 finishes and splits into another two threads, both with increased phase number $\ell=3$. 
Instance 5 contains exactly $2$ items and $2$ positions and Instance $6$ contains $3$ positions but $7$ items. Note that the involved items become even less. Right after time $t_3$, Instance $4,5,6$ run in parallel and recommend lists together.

\rrank has two aspects that one may think can lead to an unjustifiable increase of regret:
\emph{(i)} each subroutine only uses data from the first position to estimate attractiveness,
and \emph{(ii)} data collected by one subroutine is not re-used subsequently.
The second of these is relatively minor. Like many elimination algorithms, the halving of the precision means
that at most a constant factor is lost by discarding the data.
The first issue is more delicate. On the one hand, it seems distasteful not to use all available data. But the assumptions do not make it easy
to use data collected in later positions. And actually the harm may not be so great. Intuitively the cost of only using data from the
first position is greatest when the interval is large and the attractiveness varies greatly within the interval.
In this case, however, a split will happen relatively fast.

\paragraph{Running time}
The most expensive component is computing the $G$-optimal design. 
This is a convex optimisation problem and has been studied extensively (see, \citealt[\S7.5]{BV04} and \citealt{Tod16}). 
\todoc{So what is the running time for solving it??}
It is not necessary to solve the optimisation problem exactly. 
Suppose instead we find a distribution $\pi$ on $\cA$ with support at most $D(D+1)/2$ and for which
$\max_{a \in \cA} \norm{a}_{Q(\pi)^{\dagger}}^2 \leq D$. 
Then our bounds continue to hold with $d$ replaced by $D$. 
Such approximations are generally easy to find. 
For example, $\pi$ may be chosen to be a uniform distribution
on a volumetric spanner of $\cA$ of size $D$.
\ifsup See Appendix~\ref{app:vspan} for a summary on volumetric spanners. \else 
See the supplementary material for a summary of volumetric spanners. \fi
\citet{H16volumetric} 
provide a randomized algorithm that returns a volumetric spanner of size at most $O(d \log(d) \log(\abs{\cA}))$ 
with an expected running time of $O(\abs{\cA} d^2)$. 
For the remaining parts of the algorithm, the least-squares estimation is at most $O(d^3)$. 
The elimination and partitioning run in $O(\abs{\cA}d)$.
Note these computations happen only once for each instantiation. 
The update for each partition in each round is $O(d^2)$.
The total running time is $O(Ld^2\log(T)+Kd^2T)$.

\section{Regret Analysis}
\label{sec:regret analysis}

Our main theorem bounds the regret of \cref{alg:main}.

\begin{theorem}\label{thm:upper}
There exists a universal constant $C > 0$ such that
the regret bound for Algorithm \ref{alg:main} with $\delta = 1/\sqrt{T}$ satisfies 
\begin{align*}
R_T \le C K \sqrt{dT\log(LT)}\,.
\end{align*}
\end{theorem}

Let $I_\ell$ be the number of calls to \rrank{} with phase number $\ell$. 
Hence each $i \in [I_\ell]$ corresponds to a call of \rrank{} with phase number $\ell$ and the arguments are denoted by
$\cA_{\ell i}$ and $\cK_{\ell i}$.
Abbreviate $K_{\ell i} = \min \cK_{\ell i}$ for the first position of $\cK_{\ell i}$, $M_{\ell i} = |\cK_{\ell i}|$ 
for the number of positions and $\cK_{\ell i}^+ = \cK_{\ell i} \setminus \{K_{\ell i}\}$.
We also let  $K_{\ell, I_\ell+1}=K+1$ and
assume that the calls $i \in [I_\ell]$ are ordered so that
\begin{align*}
1 = K_{\ell 1} < K_{\ell 2} < \cdots < K_{\ell I_\ell} \leq K < K+1 = K_{\ell, I_\ell+1}\,.
\end{align*}
The reader is reminded that $\chi^\ast_k = \chi(A^\ast, k)$ is the examination probability of the $k$th position under the optimal list.
Let $\chi_{\ell i} = \chi_{K_{\ell i}}^\ast$ be the shorthand for the optimal examination probability of the first position in call $(\ell, i)$.
We let $\hat \theta_{\ell i}$ be the least-squares estimator computed in \cref{eq:lse} in \cref{alg:main}.  
The maximum phase number during the entire operation of the algorithm is $\ell_{\max}$.

\begin{definition}
Let $F$ be the failure event that there exists an $\ell \in [\ell_{\max}]$, $i \in [I_\ell]$ and $a \in \cA_{\ell i}$ such that
\begin{align*}
\abs{\ip{\hat{\theta}_{\ell i}, a} - \chi_{\ell i} \ip{\theta_\ast, a}} \geq \Delta_\ell
\end{align*}
or there exists an $\ell \in [\ell_{\max}]$, $i \in [I_\ell]$ and $k \in \cK_{\ell i}$ such that
$a_k^\ast \notin \cA_{\ell i}$.
\end{definition}

The first lemma shows that the failure event occurs with low probability.
The proof follows the analysis in \citep[Chap. 22]{LS18book} and is summarised in
\ifsup
\cref{app:lem:failure}.
\else 
the supplementary material.
\fi

\begin{lemma}\label{lem:failure}
$\Prob{F} \leq \delta$.
\end{lemma}

The proofs of the following lemmas are 
\ifsup
provided in \cref{app:sec:proofs of technical lemmas}.
\else 
also provided in the supplementary material.
\fi

\begin{lemma}
\label{lem:adjacent gap}
On the event $F^c$ it holds for any $\ell \in [\ell_{\max}]$, $i \in [I_\ell]$ and positions $k, k+1 \in \cK_{\ell i}$ that
$\chi_{\ell i}
(\alpha(a_k^\ast) - \alpha(a_{k+1}^\ast)) \le 8\Delta_\ell$.
\end{lemma}

\begin{lemma}
\label{lem:suboptimal gap}
On the event $F^c$ it holds for any $\ell \in [\ell_{\max}]$ and $a \in \cA_{\ell I_\ell}$ that
$\chi_{\ell I_\ell} (\alpha(a_K^\ast) - \alpha(a)) \le 8\Delta_\ell$.
\end{lemma}

\begin{lemma}\label{lem:first-subopt}
Suppose that in its $(\ell, i)$th call \rrank places item $a$ 
in position $k = K_{\ell i}$. Then, provided $F^c$ holds, 
$\chi_{\ell i} \left(\alpha(a_k^*) - \alpha(a)\right) \leq 8 M_{\ell i}$.
\end{lemma}

\begin{lemma}
\label{lem:regret on lower positions}
Suppose that in its $(\ell, i)$th call \rrank places item $a$ 
in position $k \in \cK_{\ell i}^+$. Then provided $F^c$ holds,
$\chi_{\ell i} \left(\alpha(a_k^\ast) - \alpha(a)\right) \leq 4\Delta_\ell$.
\end{lemma}

\begin{proof}[Proof of Theorem~\ref{thm:upper}]
The first step is to decompose the regret using the failure event:
\begin{align*}
R_T \leq \Prob{F} T K + \EE\left[\sind{F^c}\! \sum_{t=1}^T \sum_{k=1}^K (v(A^*, k) - v(A_t, k))\right]\,.
\end{align*}
From now on we assume that $F^c$ holds and bound the term inside the expectation.
Given $\ell$ and $i \in [I_\ell]$ let $\cT_{\ell i}$ be the set of rounds when algorithm $(\ell, i)$ is active.
Then
\begin{align}
\sum_{t=1}^T \sum_{k=1}^K (v(A^*, k) - v(A_t, k))
&= \sum_{\ell=1}^{\ell_{\max}} \sum_{i=1}^{I_\ell} R_{\ell i}\,,
\label{eq:decomp}
\end{align}
where $R_{\ell i}$ is the regret incurred during call $(\ell, i)$:
\begin{align*}
R_{\ell i} = \sum_{t \in \cT_{\ell i}} \sum_{k \in \cK_{\ell i}} (v(A^*, k) - v(A_t, k))\,.
\end{align*}
This quantity is further decomposed into the first position in $\cK_{\ell i}$, which is used for exploration, and the
remaining positions:
\begin{align*}
R_{\ell i}^{(1)} &= \sum_{t \in \cT_{\ell i}} (v(A^*, K_{\ell i}) - v(A_t, K_{\ell i}))\,. \\
R_{\ell i}^{(2)} &= \sum_{t \in \cT_{\ell i}} \sum_{k \in \cK_{\ell i}^+} (v(A^*, k) - v(A_t, k))\,.
\end{align*}
Each of these terms is bounded separately. 
For the first term we have
\begin{align}
R_{\ell i}^{(1)} 
&= \sum_{t \in \cT_{\ell i}} (v(A^*, K_{\ell i}) - v(A_t, K_{\ell i})) \nonumber \\
&= \sum_{t \in \cT_{\ell i}} \chi(A^*, K_{\ell i}) \alpha(a^*_{K_{\ell i}}) - \chi(A_t, K_{\ell i}) \alpha(A_t(K_{\ell i}))\nonumber \\
&= \sum_{t \in \cT_{\ell i}} \chi_{\ell i} \left\{\alpha(a^*_{K_{\ell i}}) - \alpha(A_t(K_{\ell i}))\right\} \nonumber \\
&\leq 8 \sum_{t \in \cT_{\ell i}} M_{\ell i} \Delta_\ell\,, \label{eq:relli1}
\end{align}
where the first equality is the definition of $R_{\ell i}^{(1)}$, the second is the definition of $v$.
The third inequality is true because event $F^c$ ensures that 
\begin{align*}
\{A_t(k) : k < K_{\ell i}\} = \{a^*_k : k < K_{\ell i}\}\,,
\end{align*}
which combined with \cref{ass:perm} shows that $\chi(A^*, K_{\ell i}) = \chi(A_t, K_{\ell i}) = \chi_{\ell i}$.
The inequality in \cref{eq:relli1} follows from \cref{lem:first-subopt}.
Moving on to the second term,
\begin{align}
R_{\ell i}^{(2)} 
&= \sum_{t \in \cT_{\ell i}} \sum_{k \in \cK_{\ell i}^+} (v(A^*, k) - v(A_t, k)) \nonumber \\
&\le \sum_{t \in \cT_{\ell i}} \sum_{k \in \cK_{\ell i}^+} \chi_k^\ast (\alpha(a^*_k) - \alpha(A_t(k))) \nonumber \\
&\leq \sum_{t \in \cT_{\ell i}} \sum_{k \in \cK_{\ell i}^+} \chi_{\ell i} (\alpha(a^*_k) - \alpha(A_t(k))) \nonumber \\
&\leq 4\sum_{t \in \cT_{\ell i}} \sum_{k \in \cK_{\ell i}^+} \Delta_\ell \label{eq:relli2} \\
&\leq 4\sum_{t \in \cT_{\ell i}} M_{\ell i} \Delta_\ell\,, \nonumber
\end{align}
where the second inequality follows from \cref{ass:min} and the third inequality follows from \cref{ass:decrease} on ranking $A^\ast$.
The inequality in \cref{eq:relli2} follows from \cref{lem:regret on lower positions} 
and the one after it from the definition of $M_{\ell i} = |\cK_{\ell i}|$.
Putting things together,
\begin{align*}
(\ref{eq:decomp}) 
&= 12 \sum_{\ell=1}^{\ell_{\max}} \sum_{i\in I_\ell} |\cT_{\ell i}| M_{\ell i} \Delta_\ell 
\le 12 K \sum_{\ell=1}^{\ell_{\max}} \max_{i\in I_\ell} |\cT_{\ell i}|  \Delta_\ell\,, \numberthis
\label{eq:sumtosplit}
\end{align*}
where we used that $ \sum_{i\in I_\ell} M_{\ell i} = K$.
To bound $|\cT_{\ell i}|$ note that, on the one hand, 
$|\cT_{\ell i}|\le T$ (this will be useful when $\ell$ is large),
while on the other hand, 
by the definition of the algorithm and the fact that the $G$-optimal design is supported 
on at most $d(d+1)/2$ points we have
\begin{align*}
\MoveEqLeft
|\cT_{\ell i}| 
\leq \sum_{a \in \cA_{\ell i}} \ceil{\frac{2 d \pi(a) \log(1/\delta_{\ell})}{\Delta_\ell^2}} \\
&\leq \frac{d(d+1)}{2} + \frac{2 d \log(1/\delta_{\ell})}{\Delta_{\ell}^2}\,.
\end{align*}
We now split to sum in \eqref{eq:sumtosplit} into two.
For $1\le \ell_0 \le \ell_{\max}$ to be chosen later,
\begin{align*}
\MoveEqLeft
\sum_{\ell=1}^{\ell_0} \max_{i\in I_\ell} |\cT_{\ell i}|  \Delta_\ell
\le
\frac{d(d+1)}{2} 
+
4 d   \log(1/\delta_{\ell_0}) 2^{\ell_0}\,,
\end{align*}
while
\begin{align*}
\sum_{\ell=\ell_0+1}^{\ell_{\max}} \max_{i\in I_\ell} |\cT_{\ell i}|  \Delta_\ell
\le
T \sum_{\ell=\ell_0+1}^{\ell_{\max}}  \Delta_\ell \le T 2^{-\ell_0}\,,
\end{align*}
hence,
\begin{align*}
(\ref{eq:decomp}) 
\le 12 K \left\{ \frac{d(d+1)}{2}  + 4 d   \log(1/\delta_{\ell_0}) 2^{\ell_0} +T 2^{-\ell_0} \right\}\,.
\end{align*}
The result is completed by optimising $\ell_0$.
\end{proof}

\section{Experiments}

\begin{figure*}[thb!]
\centering
\includegraphics[width=0.32\textwidth,height=3.5cm]{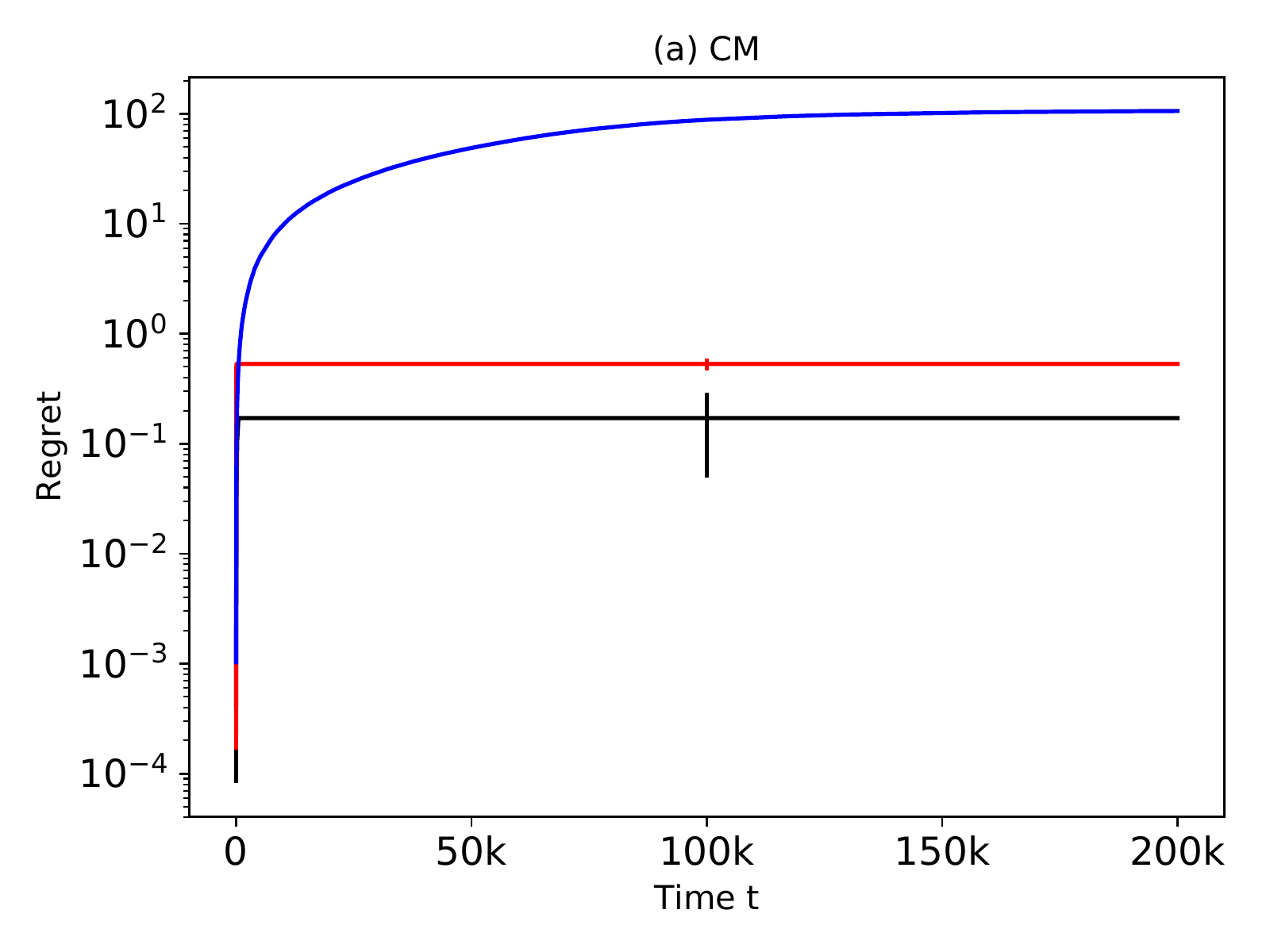}
\includegraphics[width=0.32\textwidth,height=3.5cm]{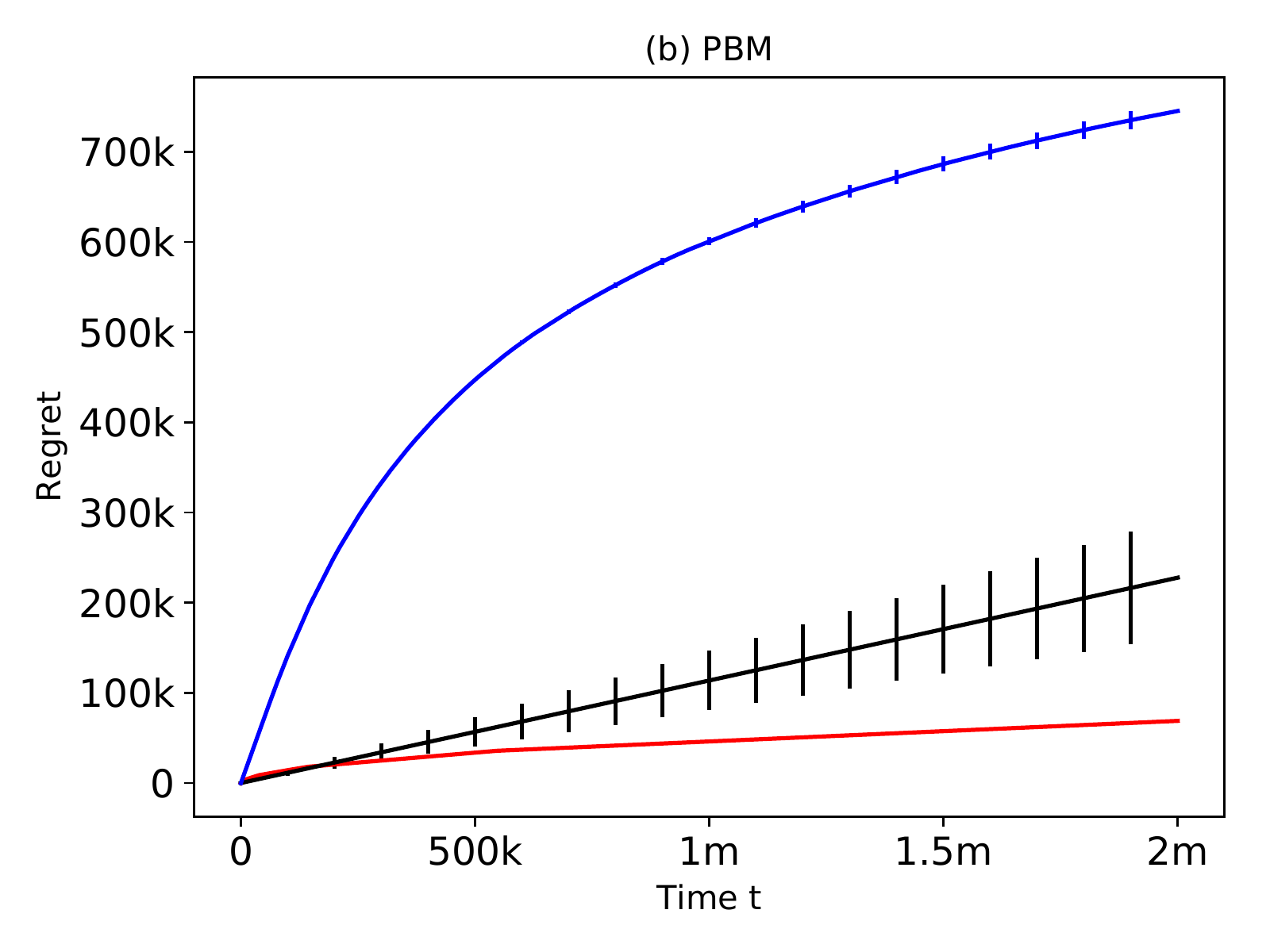}
\includegraphics[width=0.32\textwidth,height=3.5cm]{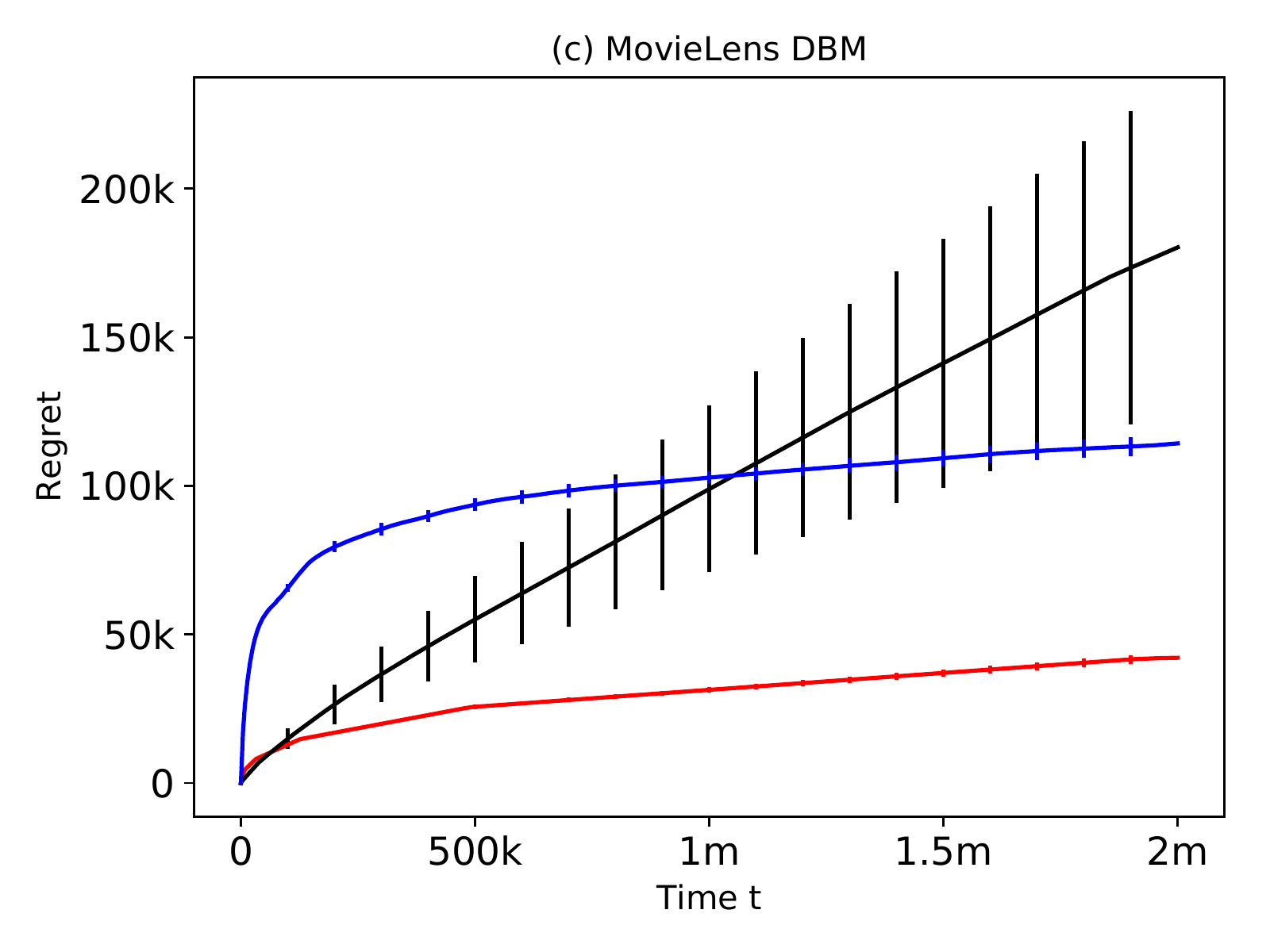}
\caption{
The figures compare \rrank{} (red) with \cascadelinucb{} (black) and \toprank{} (blue).
Subfigure (a) shows results for an environment that follows
the cascade click model (CM), while subfigure (b) does the same for the position-based click model (PBM).
On these figures, regret over time is shown (smaller is better).
In both models there are $L=10^4$ items and $K=10$ positions, and the feature space dimension is $d=5$.
Note the logarithmic scale of the $y$ axis on subfigure (a).
Subfigure (c) shows the regret over time on the MovieLens dataset with $L=10^3$, $d=5$, $K=10$.
All results are averaged over $10$ random runs. The error bars are standard errors. 
}
\vspace{-0.1cm}
\label{fig:results}
\end{figure*}

We run experiments to compare \rrank{} with \cascadelinucb{} \cite{LWZC16,ZNK16} and \toprank{} \cite{LKLS18ranking}.

\paragraph{Synthetic experiments} 
We construct environments using the cascade click model (CM) and the position-based click model (PBM) with $L=10^4$
items in $d=5$ dimension to be displayed in $K=10$ positions. 
We first randomly draw item vectors $\cL$ and weight vector $\theta_\ast$ in $d-1$ dimension with each entry a standard Gaussian variable, then normalise, add one more dimension with constant $1$, and divide by $\sqrt{2}$. 
The transformation is as follows:
\begin{align}
x\mapsto \left(\frac{x}{\sqrt{2}\norm{x}},\ \frac{1}{\sqrt{2}}\right)\,. \label{eq:transform x}
\end{align}
This transformation on both the item vector $x \in \cL \subset \RR^d$ and weight vector $\theta_\ast$ is to guarantee the attractiveness $\ip{\theta_\ast, x}$ of each item $x$ lies in $[0,1]$.
The position bias for PBM is set as $\left(1,\frac{1}{2},\frac{1}{3},\ldots,\frac{1}{K}\right)$ which is often adopted in applications \cite{wang2018position}. 
The evolution of the regret as a function of time is shown in \cref{fig:results}(a)(b). The regrets at the end and total running times are given in
\ifsup
\cref{app:sec:quantity results for experiments}.
\else 
the supplementary material.
\fi

\cascadelinucb{} is best in CM but worst in PBM because of its modelling bias. 
\toprank{} takes much longer time to converge than either 
\cascadelinucb{} or \rrank since it neither exploits the specifics of the click model, nor does it use the linear structure.

\paragraph{MovieLens dataset} We use the  
$20m$ MovieLens dataset \cite{harper2016movielens} 
which contains $20$ million ratings for $2.7\times 10^4$ movies by $1.38 \times 10^5$ users. We extract $L=10^3$ movies with most ratings and $1.1 \times 10^3$ users who rate most and randomly split the user set to two parts, $U_1$ and $U_2$ with $\abs{U_1}=100$ and $\abs{U_2}=10^3$. We then use the rating matrix of users in $U_1$ to derive feature vectors with $d=5$ for all movies by singular-value decomposition (SVD). The resulting feature vectors $\cL$ are also processed as \eqref{eq:transform x}. The true weight vector $\theta_\ast$ is computed by solving the linear system of $\cL$ w.r.t. the rating matrix of $U_2$. The environment is the document-based click model (DBM) with $\cL$ and $\theta_\ast$ and we set $K=10$. The performances are measured in regret, as shown in \cref{fig:results}(c). As can be seen, \rrank learns faster than the other two algorithms. Of these two algorithms, the performance of \cascadelinucb saturates: this is due to its incorrect bias. 

\section{Discussion}\label{sec:discussion}

\paragraph{Assumptions}
Our assumptions are most closely related to the work by \citet{LKLS18ranking} and \citet{ZTG17}.  
The latter work also assumes a factored model where the probability of clicking on an item factors into an examination probability and an attractiveness function. 
None of these works make use of features to model the attractiveness of items: 
They are a special case of our model when we set the features of items to be orthogonal to each other (in particular, $d=L$).
Our assumptions on the examination probability function are weaker than those by \citet{ZTG17}.
Despite this, our regret upper bound is better by a factor of $K$ (when setting $d=L$) and the analysis is also simpler.
The paper by \citet{LKLS18ranking} 
does not assume a factored model, but instead places assumptions directly on $v$.
They also assume a specific behaviour of the $v$ function under pairwise exchanges that is not required here. 
Their assumptions are weaker in the sense that they do not assume the probability of clicking on position $k$ only depends on the identities of the items in positions $[k-1]$ and the attractiveness of the item in position $k$. 
On the other hand, they do assume a specific behaviour of the $v$ function under pairwise exchanges that is not required by our analysis.
It is unclear which set of these assumptions is preferable.

\vspace{-0.3cm}
\paragraph{Lower bounds}
In the orthogonal case where $d = L$ the lower bound in \cite{LKLS18ranking} provides an example where the regret is
at least $\Omega(\sqrt{T K L})$. For $d \leq L$, the standard techniques for proving lower bounds for linear bandits can be used
to prove the regret is at least $\Omega(\sqrt{d TK})$, which except for logarithmic terms means our upper bound is suboptimal by a factor of
at most $\sqrt{K}$. We are not sure whether either the lower bound or the upper bound is tight.

\vspace{-0.3cm}
\paragraph{Open questions}
Only using data from the first position seems suboptimal, but is hard to avoid without making additional assumptions. 
Nevertheless, we believe a small improvement should be possible here.
Another natural question is how to deal with the situation when the set of available items is changing. In practice this happens in many applications, either
because the features are changing or because new items are being added or removed.
Other interesting directions are to use weighted least-squares estimators to exploit the low variance when the examination probability and attractiveness are small.
Additionally one can use a generalised linear model instead of the linear model to model the attractiveness function, which may be analysed
using techniques developed by \citet{FCGS10} and \citet{JBNW17}.
Finally, it could be interesting to generalise to the setting where item vectors are sparse (see \citealt{APS12} and \citealt[Chap. 23]{LS18book}).

\bibliography{ref}

\ifsup
\clearpage

\appendix 
\section{Proof of Lemma~\ref{lem:failure}}\label{app:lem:failure}

In what follows, we add the index $(\ell,i)$ to any symbol used in the algorithm to indicate the value that it takes in the $(\ell,i)$ call. For example,  $\cD_{\ell i}$ denotes the data multiset collected in the $(\ell,i)$ call, $T_{\ell i}(a)$ be the value computed in \cref{eq:allocchoice}, etc.

Fix $\ell\ge 1$
and let $F_\ell$ be the failure event that there exists an $i \in [I_\ell]$ and $a \in \cA_{\ell i}$ such that
\begin{align*}
\abs{\ip{\hat{\theta}_{\ell i}, a} - \chi_{\ell i} \ip{\theta_\ast, a}} \geq \Delta_\ell\,.
\end{align*}

Let $E_\ell$ be the event that  for any $i \in [I_\ell]$,
 the examination probability on the first position of the call $(\ell, i)$ is $\chi_{\ell i}$.
For the argument that follows, let us assume that $E_\ell$ holds.

By our modelling assumptions (\cref{eq:click,eq:clickfactor,eq:clicklinear}),
for any $(\beta, \zeta) \in \cD_{\ell i}$,
\begin{align*}
	\zeta = \ip{\chi_{\ell i}\theta_\ast, \beta} + \eta_{(\beta, \zeta)}\,,
\end{align*}
where $\{\eta_{(\beta, \zeta)}\}_{(\beta, \zeta)}$ is a conditionally $1/2$-subgaussian sequence. 

Define the Gram matrix $Q$ for any probability mass function $\pi:\cA \to [0,1]$, $\sum_{a \in \cA} \pi(a) = 1$, as $Q(\pi) = \sum_{a \in \cA} \pi(a) aa^\top$. 
By the Kiefer-Wolfowitz theorem \cite{KW60},
\begin{align*}
	\max_{a \in \cA_{\ell i}} \norm{a}_{Q(\pi_{\ell i})^{\dagger}}^2 = \rank(\cA) \le d\,,
\end{align*}
where $Q^\dagger$ denotes the Moore-Penrose inverse of $Q$. Then, by \cref{eq:allocchoice}, 
\begin{align*}
	V_{\ell i} &=\sum_{a \in \cA_{\ell i}} T_{\ell i}(a) aa^\top 
	\succeq \frac{d}{2\Delta_\ell^2}\log\left(\frac{\abs{\cA_{\ell i}} }{\delta_\ell}\right)Q(\pi_{\ell i})\,,
\end{align*}
where $P\succeq Q$ denotes that $P$ precedes $Q$ in the Loewner partial ordering of positive semi-definite (symmetric) matrices.
This implies that 
\begin{align*}
\norm{a}_{V_{\ell i}^{\dagger}}^2 
&\le \frac{2\Delta_\ell^2}{d} \frac{1}{\log\left(\frac{\abs{\cA_{\ell i}} }{\delta_\ell}\right)} \norm{a}_{Q(\pi_{\ell i})^{\dagger}}^2 \\
&\le 2\Delta_\ell^2 \frac{1}{\log\left(\frac{\abs{\cA_{\ell i}} }{\delta_\ell}\right)} \,.
\end{align*}
Rearranging shows that
\begin{align}
\Delta_\ell &\ge \sqrt{\frac{1}{2} \norm{a}_{V_{\ell i}^{\dagger}}^2 \log\left(\frac{\abs{\cA_{\ell i}} }{\delta_\ell}\right)}\,. \label{eq:Delta ell}
\end{align}
Now note that
\begin{align*}
	&\ip{\hat{\theta}_{\ell i} - \chi_{\ell i} \theta_\ast, a}\\
    =& \ip{V_{\ell i}^{\dagger}\sum_{(\beta, \zeta) \in \cD_{\ell i}} \beta \zeta -  \chi_{\ell i} \theta_\ast, a}\\
    =& \ip{V_{\ell i}^{\dagger}\sum_{(\beta, \zeta) \in \cD_{\ell i}} \beta (\beta^{\top} \theta_\ast \chi_{\ell i} + \eta_{(\beta, \zeta)}) -  \chi_{\ell i} \theta_\ast, a}\\
    =&\chi_{\ell i} \ip{(V_{\ell i}^{\dagger} V_{\ell i}-I) \theta_\ast, a} + \ip{V_{\ell i}^{\dagger}\sum_{(\beta, \zeta) \in \cD_{\ell i}}\beta \eta_{(\beta, \zeta)}, a}\\
    =&\sum_{(\beta, \zeta) \in \cD_{\ell i}} \ip{V_{\ell i}^{\dagger} \beta, a} \ \eta_{(\beta, \zeta)}\,. \numberthis
    \label{eq:noise}
\end{align*}
The last equality follows from $I - V_{\ell i}^{\dagger} V_{\ell i}$ is the orthogonal projection on the kernel of $V_{\ell i}$, which is the 
orthogonal complement of $\cA_{\ell i}$, and thus will map each $a \in \cA_{\ell i}$ to the zero vector.
Then, for any $a \in \cA_{\ell i}$,
\begin{align*}
	&\PP\left(\abs{\ip{\hat{\theta}_{\ell i} - \chi_{\ell i} \theta_\ast, a}} \ge \Delta_\ell \right) \\
    &\le \PP\left(\abs{\ip{\hat{\theta}_{\ell i} - \chi_{\ell i} \theta_\ast, a}} \ge \sqrt{\frac{1}{2} \norm{a}_{V_{\ell i}^{\dagger}}^2 \log\left(\frac{\abs{\cA_{\ell i}} }{\delta_\ell}\right)} \right)\\
    &\le \frac{2\delta_\ell}{\abs{\cA_{\ell i}}}\,.
\end{align*}
The first inequality is by \cref{eq:Delta ell}. 
The second inequality is by \cref{eq:noise}, 
the concentration bound on conditional subgaussian sequences \citep[Lemma 5.2 and Theorem 5.1]{LS18book},
and $\sum_{(\beta, \zeta) \in \cD_{\ell i}}\ip{V_{\ell i}^{\dagger}\beta, a}^2 = \norm{a}_{V_{\ell i}^{\dagger}}^2$.
Thus with probability at least $1 - 2\delta_\ell$,
\begin{align*}
	\abs{\ip{\hat{\theta}_{\ell i} - \chi_{\ell i} \theta_\ast, a}} \le \Delta_\ell
\end{align*}
holds for any $a \in \cA_{\ell i}$ and thus from $I_\ell\le K$, we get that 
\begin{align}
\Prob{F_\ell \cap E_\ell} \le 2K\delta_\ell\,.
\label{eq:fllell}
\end{align}

Now we prove by induction on $\ell$ that on the complementer of $F_{1:\ell-1}=F_1\cup \dots \cup F_{\ell-1}$  (with $F_{1:0}=\emptyset$)
the following hold true: 
{\em (i)} 
the examination probability on the first position of the call $(\ell, i)$ is $\chi_{\ell i}$ for any $i \in [I_\ell]$;
{\em (ii)} 
$a_{K_{\ell I_\ell}}^\ast, \ldots, a_K^\ast$ are the $M_{\ell I_{\ell}}$ best items in $\cA_{\ell I_{\ell}}$ and
that 
{\em (iii)} for any $i,j\in [I_{\ell}]$, $i<j$, and $a\in \cA_{\ell i}$, $a'\in \cA_{\ell j}$, it holds that $\alpha(a)<\alpha(a')$
(note that {\em (ii)} and {\em (iii)} just mean that the algorithm does not make a mistake when it eliminates items or splits blocks).
The claim is obviously true for $\ell=1$.
In particular, 
the examination probability on the first position of the call $(\ell=1, i=1)$ is $\chi_{1,1}$ by \cref{ass:perm}. 

Now, let $\ell\ge 1$ and
suppose $F_{1:\ell}$ does not hold. 
If $\ip{\hat{\theta}_{\ell i}, a} - \ip{\hat{\theta}_{\ell i}, a'} \ge 2\Delta_{\ell}$ 
for some $a, a' \in \cA_{\ell i}$ and $i\in [I_\ell]$, then by {\em (i)} of the induction hypothesis,
\begin{align*}
	\chi_{\ell i} \ip{\theta_\ast, a} &> \ip{\hat{\theta}_{\ell i}, a} - \Delta_{\ell} \\
    &\ge \ip{\hat{\theta}_{\ell i}, a'} + \Delta_{\ell} > \chi_{\ell i} \ip{\theta_\ast, a'}\,,
\end{align*}
thus $\alpha(a) > \alpha(a')$. 

If $a \in \cA_{\ell I_{\ell}}$ is eliminated at the end of call $(\ell, I_\ell)$, there exists $m = M_{\ell I_{\ell}}$ different items $b_1, \ldots, b_m \in \cA_{\ell I_{\ell}}$ such that $\ip{\hat{\theta}_{\ell i}, b_j} - \ip{\hat{\theta}_{\ell i}, a} \ge 2\Delta_{\ell}$ for all $j \in [m]$. Thus $\alpha(b_j) > \alpha(a)$ for all $j \in [m]$. Since, by induction, $a_{K_{\ell I_\ell}}^\ast, \ldots, a_K^\ast$ are $m$ best items in $\cA_{\ell I_{\ell}}$, 
 then $\alpha(a) < \alpha(a_K^\ast)$. This shows that {\em (ii)} will still hold for $\cA_{\ell+1, I_{\ell+1}}$.

If there is a split $\cA_1, \ldots, \cA_p$ and $\cK_1, \ldots, \cK_p$ on $\cA_{\ell i}$ and $\cK_{\ell i}$ 
by the algorithm, $\ip{\hat{\theta}_{\ell i}, a} - \ip{\hat{\theta}_{\ell i}, a'} \ge 2\Delta_{\ell}$ for any $a \in \cA_j, a'\in \cA_{j+1}, j \in [p-1]$. 
Then $\alpha(a) > \alpha(a')$.
So the better arms are put at higher positions, 
which combined with that {\em (iii)} holds at stage $\ell$ shows that
 {\em (iii)} will still continue to hold for $\ell+1$.

Finally, it also follows that 
$\chi_{\ell+1, i} = \chi_{K_{\ell+1,i}}^\ast$ is the examination probability of the first position for any call $(\ell+1, i)$ of phase $\ell+1$, showing that {\em (i)} also continues to hold for phase $\ell+1$.

From this argument it follows that $F_{1:\ell-1}^c \subset E_\ell$ holds for all $\ell\ge 1$.
Then,
\begin{align*}
F 
&             = (F_{1:1}\cap F_{1:0}^c) \cup (F_{1:2} \cap F_{1:1}^c) \cup (F_{1:3} \cap F_{1:2}^c) \cup \dots\\
& \subset (F_{1:1} \cap E_1) \cup (F_{1:2} \cap E_2) \cup (F_{1:3} \cap E_{3}) \cup \dots\,.
\end{align*}
Taking probabilities and using \eqref{eq:fllell}, we get
\begin{align*}
\Prob{F} =  \sum_{\ell\ge 1} \Prob{ F_{1:\ell} \cap E_\ell } \le \delta\,,
\end{align*}
finishing the proof.

\section{Volumetric Spanners}\label{app:vspan}

A volumetric spanner of compact set $\cK \subset \R^d$ is a finite set $S = \{x_1,\ldots,x_n\} \subseteq \cK$ such that
\begin{align*}
\cK \subseteq \cE(S) = \set{\sum_{i=1}^n \alpha_i x_i : \norm{\alpha}_2 \leq 1}\,.
\end{align*}
Let $\pi$ be a uniform distribution on $S$ and 
\begin{align*}
Q = \sum_{i=1}^n \pi(x_i) x_i x_i^\top\,.
\end{align*}
If $S$ is a volumetric spanner of $\cK$,
for any $x \in \cK$ it holds that $\norm{x}_{Q^{\dagger}}^2 \leq n$.
To see this let $U \in \R^{d \times n}$ be the matrix with columns equal to the elements in $S$, which means that $Q = U U^\top/n$.
Since $x \in \cK$ there exists an $\alpha \in \R^n$ with $\norm{\alpha}_2 \leq 1$ such that $x = U\alpha$.
Then
\begin{align*}
x^\top Q^{\dagger} x
&= n\alpha^\top U^\top (U U^\top)^{\dagger} U \alpha  \\
&= n\alpha^\top U^{\dagger} U \alpha \\
&\leq n \norm{\alpha}_2^2 \\
&\leq n\,.
\end{align*}
Any compact set admits a volumetric spanner of size $n \leq 12d$, hence by \cref{eq:gopt}, a volumetric spanner is a ``$12$-approximation'' to the $G$-optimal design problem. 
For finite $\cK$ with $n$ points in it, the and for $\epsilon>0$ fixed, a spanner of size $12(1+\epsilon)d$ can be computed in $O(n^{3.5} + d n^3 + n d^3)$ time \citep[Theorem 3]{H16volumetric}. 
\todoc{There is a $\log(1/\epsilon)$, which is suppressed.}

\section{Proofs of Technical Lemmas}
\label{app:sec:proofs of technical lemmas}

\begin{proof}[Proof of \cref{lem:adjacent gap}]
Let $F^c$ hold.
Since $\Delta_1 = 1/2$, the result is trivial for $\ell = 1$. 
Suppose $\ell > 1$, the lemma holds for all $\ell' < \ell$ and that there exists a pair $k, k+1 \in \cK_{\ell i}$ satisfying
$\chi_{\ell i}
(\alpha(a_k^\ast) - \alpha(a_{k+1}^\ast)) > 8\Delta_\ell$.
Let $(\ell-1,j)$ be the parent of $(\ell, i)$, which satisfies $a_k^\ast, a_{k+1}^\ast \in \cA_{\ell i} \subseteq \cA_{\ell-1,j}$.
Since $K_{\ell - 1,j} \leq K_{\ell i}$ it follows from \cref{ass:decrease} and the definition of $F$ that $\chi_{\ell-1,j} \ge \chi_{\ell i}$ and hence
\begin{align*}
\chi_{\ell-1,j} \left(\alpha(a_k^\ast) - \alpha(a_{k+1}^\ast)\right) > 8\Delta_\ell = 4 \Delta_{\ell - 1}\,,
\end{align*}
where we used the definition of $\Delta_\ell = 2^{-\ell}$.
Given any $m, n \in \cK_{\ell-1,j}$ with $m \le k < k+1 \le n$ we have
\begin{align*}
\ip{\hat{\theta}_{\ell-1, j}, a_m^\ast} 
&\ge \chi_{\ell-1,j} \alpha(a_m^\ast) - \Delta_{\ell-1} \\
&\ge \chi_{\ell-1,j} \alpha(a_k^\ast) - \Delta_{\ell-1} \\
&> \chi_{\ell-1,j} \alpha(a_{k+1}^\ast) + 3\Delta_{\ell-1} \\
&\ge \chi_{\ell-1,j} \alpha(a_n^\ast) + 3\Delta_{\ell-1} \\
&\ge \ip{\hat{\theta}_{\ell-1, j}, a_n^\ast} + 2\Delta_{\ell-1}\,.
\end{align*}
The first and fifth inequalities are because $F$ does not hold. The third inequality is due to induction assumption on phase $\ell-1$.
Hence by the definition of the algorithm the items $a_k^\ast$ and $a_{k+1}^\ast$ will be split into different partitions by
the end of call $(\ell-1,j)$, which is a contradiction.
\end{proof}

\begin{proof}[Proof of \cref{lem:suboptimal gap}]
We use the same idea as the previous lemma.
Let $F^c$ hold.
The result is trivial for $\ell = 1$. Suppose $\ell > 1$, the lemma holds for $\ell' < \ell$ and there exists an $a \in \cA_{\ell I_\ell}$ satisfying
$\chi_{\ell I_\ell} (\alpha(a_K^\ast) - \alpha(a)) > 8\Delta_\ell$.
By the definition of the algorithm and $F$ does not hold, $a, a_K^\ast \in \cA_{\ell-1,I_{\ell-1}}$ and hence
\begin{align*}
\chi_{\ell-1,I_{\ell-1}} \left(\alpha(a_K^\ast) - \alpha(a)\right) > 4\Delta_{\ell-1} \,.
\end{align*}
For any $m \in \cK_{\ell-1,I_{\ell-1}}$ with $m \leq K$ it holds that
\begin{align*}
\ip{\hat{\theta}_{\ell-1,I_{\ell-1}}, a_m^\ast} 
&\ge \chi_{\ell-1,I_{\ell-1}} \alpha(a_m^\ast) - \Delta_{\ell-1} \\
&\ge \chi_{\ell-1,I_{\ell-1}} \alpha(a_K^\ast) - \Delta_{\ell-1} \\
&> \chi_{\ell-1,I_{\ell-1}} \alpha(a) + 3\Delta_{\ell-1} \\
&\ge \ip{\hat{\theta}_{\ell-1,I_{\ell-1}}, a} + 2\Delta_{\ell-1} \,.
\end{align*}
Hence there exist at least $M_{\ell-1,I_{\ell-1}}$ items $b \in \cA_{\ell-1,I_{\ell-1}}$ for which $\ip{\hat \theta_{\ell-1,I_{\ell-1}}, b - a} \geq 2\Delta_{\ell - 1}$. 
But if this was true then by the definition of the algorithm 
(cf. line~7) 
item $a$ would have been eliminated by the end of call $(\ell-1,I_{\ell-1})$, which is a contradiction. 
\end{proof}

\begin{proof}[Proof of \cref{lem:first-subopt}]
Let $F^c$ hold.
Suppose that $i < I_\ell$ and abbreviate $m = M_{\ell i}$. Since $F$ does not hold it follows that $a \in \{a^*_k, \ldots,a^*_{k+m-1}\}$.
By \cref{lem:adjacent gap},
\begin{align*}
\chi_{\ell i} \left(\alpha(a_k^*) - \alpha(a)\right)
&\leq \chi_{\ell i} \left(\alpha(a_k^\ast) - \alpha(a_{k+m-1}^\ast)\right) \\
&= \sum_{j=0}^{m-2} \chi_{\ell i} \left(\alpha(a^\ast_{k+j}) - \alpha(a^\ast_{k+j+1})\right) \\
&\leq 8(m-1) \Delta_\ell\,.
\end{align*}
Now suppose that $i = I_\ell$. Then by \cref{lem:suboptimal gap} and the same argument as above,
\begin{align*}
&\chi_{\ell i} \left(\alpha(a_k^*) - \alpha(a)\right) \\
&\qquad = \chi_{\ell i} \left(\alpha(a_K^*) - \alpha(a)\right) + \chi_{\ell i} \left(\alpha(a_k^*) - \alpha(a_K^*)\right) \\
&\qquad \leq 8m\Delta_\ell\,. 
\end{align*}
The claim follows by the definition of $m$.
\end{proof}

\begin{proof}[Proof of \cref{lem:regret on lower positions}]
The result is immediate for $\ell = 1$. From now on assume that $\ell > 1$ and let $(\ell-1,j)$ be the parent of $(\ell, i)$.
Since $F$ does not hold, $\{a^*_m : m \in \cK_{\ell i}\} \subseteq \cA_{\ell i}$. 
It cannot be that $\ip{\hat \theta_{\ell-1,j}, a^*_m - a} > 0$ for all $m \in \cK_{\ell i}$ with $m \leq k$, since this would
mean that there are $k-K_{\ell i}+2$ items that precede item $a$ and hence item $a$ would not be put in position $k$ by the algorithm.
Hence there exists an $m \in \cK_{\ell i}$ with $m \leq k$ such that $\ip{\hat \theta_{\ell-1,j}, a^*_m - a} \leq 0$ and
\begin{align*}
\chi_{\ell i} (\alpha(a_k^\ast) - \alpha(a))
&\leq \chi_{\ell i} (\alpha(a_m^\ast) - \alpha(a)) \\
&\le \chi_{\ell-1, j} (\alpha(a_m^\ast) - \alpha(a)) \\
&\leq \ip{\hat \theta_{\ell - 1,j}, a_m^\ast - a} + 2\Delta_{\ell - 1} \\
&\leq 2\Delta_{\ell - 1} = 4\Delta_\ell\,,
\end{align*}
which completes the proof.
\end{proof}

\begin{table}[thb!]
\centering
\begin{tabular}{l|r|r|r}
 & \rrank{} & \cascadelinucb{} &\toprank{}\\
 \hline
 CM &$0.53$ &$0.17$ &$106.10$\\
\hline
PBM &$68,943$ &$227,736$ &$745,177$\\
\hline
ML &$42,157$ &$180,256$ &$114,288$\\
\hline
\end{tabular}
\caption{The total regret under (a) CM (b) PBM and (c) ML.
The number shown are computed by taking the average over the $10$ random runs.}
\label{tab:reg}
\end{table}

\begin{table}[thb!]
\centering
\begin{tabular}{l|r|r|r}
Time (s) & \footnotesize \rrank & \footnotesize \cascadelinucb & \footnotesize \toprank \\
\hline
CM &$51$ &$411$ &$176,772$\\
\hline
PBM &$310$ &$4,147$ &$367,509$\\
\hline
ML &$234$ &$916$ &$4,868$\\
\hline
\end{tabular}
\caption{The total running time of the compared algorithms in seconds (s). The results are averaged over $10$ random runs.}
\label{tab:time}
\end{table}

\section{Quantity Results for Experiments}
\label{app:sec:quantity results for experiments}

The regrets of \cref{fig:results} at the end are given in the \cref{tab:reg}, while total running times (wall-clock time) are shown in \cref{tab:time}. The experiments are run on Dell PowerEdge R920 with CPU of Quad Intel Xeon CPU E7-4830 v2 (Ten-core 2.20GHz) and memory of 512GB.

\fi

\end{document}